\documentclass[10pt,twocolumn,letterpaper]{article}

\usepackage{cvpr}
\usepackage{times}
\usepackage{graphicx}
\usepackage{amsmath}
\usepackage{amssymb}
\usepackage{color}
\usepackage{epsfig}
\usepackage{bm}
\usepackage{enumerate}
\usepackage{multirow}
\usepackage{amsthm}
\usepackage{array}
\usepackage{bigstrut}
\usepackage{algorithmic}
\usepackage{algorithm}
\usepackage{arydshln}
\usepackage{booktabs}
\usepackage[caption=false,font=footnotesize]{subfig}

\newcommand{\comment}[1]{}
\newcommand{\deflen}[2]{%
    \expandafter\newlength\csname #1\endcsname
    \expandafter\setlength\csname #1\endcsname{#2}%
}
\graphicspath{{figures/}}
\newtheorem{theorem}{Theorem}

\usepackage[pagebackref=true,breaklinks=true,letterpaper=true,colorlinks,bookmarks=false]{hyperref}

\cvprfinalcopy 

\def\cvprPaperID{15} 

\ifcvprfinal\fi
\begin{document}

\title{Deep Graph Laplacian Regularization for Robust Denoising of Real Images}

\author{Jin Zeng$^{1*}$\hspace{35pt}Jiahao Pang$^1$\thanks{Both authors contributed equally to this work. Jiahao Pang is the corresponding author.}\hspace{35pt}Wenxiu Sun$^1$\hspace{35pt}Gene Cheung$^2$\\
$^1$SenseTime Research\hspace{30pt}$^2$Department of EECS, York University\\
{\tt\small \{zengjin,\hspace{3pt}pangjiahao,\hspace{3pt}sunwenxiu\}@sensetime.com, genec@yorku.ca}
}

\maketitle
\thispagestyle{empty}

\begin{abstract}
Recent developments in deep learning have revolutionized the paradigm of image restoration. 
However, its applications on real image denoising are still limited, due to its sensitivity to training data and the complex nature of real image noise.
In this work, we combine the robustness merit of model-based approaches and the learning power of data-driven approaches for real image denoising.
Specifically, by integrating graph Laplacian regularization as a trainable module into a deep learning framework, we are less susceptible to overfitting than pure CNN-based approaches, achieving higher robustness to small datasets and cross-domain denoising.
First, a sparse neighborhood graph is built from the output of a convolutional neural network (CNN). 
Then the image is restored by solving an unconstrained quadratic programming problem, using a corresponding graph Laplacian regularizer as a prior term. 
The proposed restoration pipeline is fully differentiable and hence can be end-to-end trained. 
Experimental results demonstrate that our work is less prone to overfitting given small training data. 
It is also endowed with strong cross-domain generalization power, outperforming the state-of-the-art approaches by a remarkable margin.
\end{abstract}

\vspace{-11pt}
\section{Introduction}\label{sec:intro}


Image denoising is the most fundamental image restoration problem, which has been studied for decades. 
In order to regularize its ill-posed nature, a large body of works adopt \emph{signal priors}. 
By adopting a certain image model, one assumes that the original image should induce a small value for a given \emph{model-based} signal prior.
Representative priors in the literature include non-local self-similarity \cite{buades2005non}, total variation (TV) prior \cite{rudin1992nonlinear}, sparsity prior \cite{elad2006image}, graph Laplacian regularizer \cite{pang2017graph}, {\it etc}.
However, these works place their emphases to the removal of additive white Gaussian noise (AWGN), which is unrealistic and limits their applications in practice.
In the real world, image noise stems from multiple sources, {\it e.g.}, thermal noise, shot noise, dark current noise, making it much more sophisticated than the ideal AWGN.

\deflen{showidth}{75pt}
\deflen{gauinter}{0pt}
\begin{figure}[t]
\centering
        \subfloat[Noise Clinic.]{\includegraphics[width=\showidth]{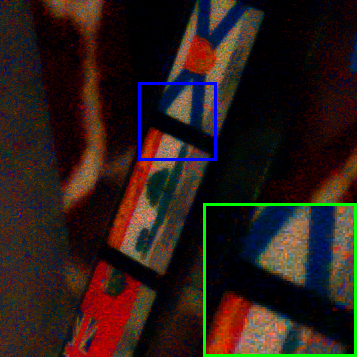}\label{fig:intro_nc}}\hspace{\gauinter}
        \subfloat[CDnCNN.]{\includegraphics[width=\showidth]{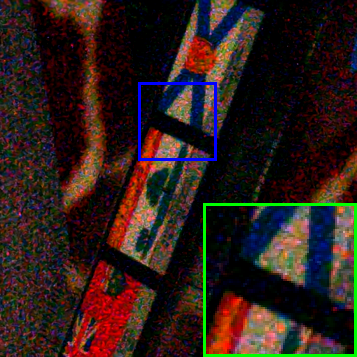}\label{fig:intro_dn}}\hspace{\gauinter}
        \subfloat[DeepGLR.]{\includegraphics[width=\showidth]{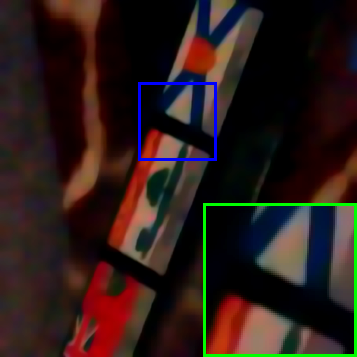}\label{fig:intro_ours}}
 \caption{Results of real image denoising. (a) Noise Clinic (model-based); (b) CDnCNN (data-driven); (c) DeepGLR (proposed). Our method and CDnCNN are trained for Gaussian denoising.}
 \label{fig:intro}
\end{figure}

Recent developments in \emph{deep learning} have revolutionized the aforementioned model-based paradigm in image restoration. 
Thanks to the strong learning capacity of convolutional neural networks (CNN) to capture image characteristics, CNN-based approaches have achieved the state-of-the-art performance in Gaussian denoising, {\it e.g.}, \cite{zhang2017beyond,vemulapalli2016deep,tai2017memnet}. 
However, the application of deep learning models on real noise removal remains quite challenging. 
Unlike model-based approaches, CNN-based approaches are data-driven. 
To learn a CNN for real image noise removal, thousands of real noisy images and their noise-free versions are necessary to characterize the correspondence between the corrupted images and the ground-truths \cite{zhu2016noise}.
Unfortunately, acquiring the noise-free images is non-trivial \cite{xu2017multi,chen2018learning}, leading to limited amount of training data. 
In this case, a purely data-driven approach is prone to overfit to the particular characteristics of the training data. 
It fails on test images with statistics different from the training images \cite{mccann2017convolutional}, {\it e.g.}, Figure\,\ref{fig:intro_dn} showcases the result of a pure data-driven approach trained for a different domain.
	
Differently, model-based denoising approaches rely on basic assumptions about the original images, which ``encode'' assumed image characteristics. 
Without the notion of training, the performance of model-based denoising is generally more robust than data-driven approaches when facing the heterogeneity of natural images \cite{dong2018denoising}. 
However, the assumed characteristics may not perfectly hold in the real world, hindering their performance and flexibility in practice \cite{milanfar2013tour}, {\it e.g.}, the denoising result of Figure\,\ref{fig:intro_nc}.

To achieve robust denoising of real images, in this paper we combine the robustness merit of model-based approaches and the powerful learning capacity of data-driven approaches. 
We achieve this goal by incorporating the graph Laplacian regularizer---a simple yet effective image prior for image restoration tasks---into a deep learning framework. 
Specifically, we train a CNN which takes as input a real noisy image and outputs a set of feature maps. 
Subsequently, a neighborhood graph is built from the output features. 
The image is then denoised by solving an unconstrained quadratic programming (QP) problem, assuming that the underlying true image induces a small value of graph Laplacian regularizer. 
Figure\,\ref{fig:intro_ours} shows the denoising result of our approach, one may clearly see its superiority to the competing methods.

The contributions of our work are as follows:
%
\begin{enumerate}[(i)]
\item We are the first in literature to incorporate the widely used graph Laplacian regularizer into deep neural networks as a \textit{fully-differentiable} layer, extracting underlying features of the input noisy images and boosting the performance of the subsequent restoration.
\item Our architecture couples the strong graph Laplacian regularization layer---an adaptive low-pass linear filter regardless of the training data---with a light-weight CNN for pre-filtering, making our approach less susceptible to overfitting. 
Moreover, by constraining the regularization weight to prevent steep local minimum, our pipeline is provably {\it numerical stable}.
\item Experimentation shows that, our approach achieves \emph{robust} real noise removal in terms of two perspectives. 
Given small amount of training data, our proposal outperforms CNN-based approaches by avoiding overfitting. 
Secondly, it exhibits strong cross-domain generalization ability, {\it e.g.},  our framework training for Gaussian denoising performs reasonably well on real image denoising.
\end{enumerate}

We call our proposal \emph{deep graph Laplacian regularization}, or DeepGLR for short.
This paper is organized as follows. 
Related works are reviewed in Section\;\ref{sec:related}. 
We then present our DeepGLR denoising framework combining CNN and a differentiable graph Laplacian regularization layer in Section\;\ref{sec:glr}. 
Section\;\ref{sec:results} presents the experimentation and Section\;\ref{sec:conclusion} concludes our work.

\section{Related Works}\label{sec:related}
We first review several deep learning models for image restoration while focusing on image denoising.
We then turn to the review of several representative signal priors, with a focus on graph Laplacian regularization. 
We also briefly review a few works on graph learning.


\textbf{CNN-based image denoising}:
CNN-based approaches were first popularized in high-level vision tasks, {\it e.g.}, classification \cite{krizhevsky2012imagenet} and detection \cite{ouyang2013joint}, then gradually penetrated into low-level restoration tasks such as image denoising \cite{zhang2017beyond}, super-resolution \cite{dong2014learning}, and non-blind deblurring \cite{xu2014deep}. 
To address the problem of Gaussian noise removal with CNN, Zhang {\it et~al.} \cite{zhang2017beyond} utilize residual learning and batch normalization to build a deep architecture, which provides state-of-the-art results.
In \cite{jain2009natural}, Jain {\it et~al.} propose a simple network for natural image denoising and relate it to Markov random field (MRF) methods. 
To build a CNN capable of handling several noise levels, Vemulapalli {\it et~al.} \cite{vemulapalli2016deep} employ conditional random field (CRF) for regularization. 
Other related works on denoising with CNN includes \cite{mao2016image,tai2017memnet,zhang2018ffdnet,lefkimmiatis2018universal,chen2018deep}, {\it etc}.
Despite their good performance, these approaches focuses on Gaussian denoising and have strong dependency on the training data. 
For effective real image denoising and enhancing, Chen {\it et al.} \cite{chen2018learning} train a CNN to directly  perform restoration on raw image data.
Differently, our DeepGLR enhances the robustness of the denoising pipeline, so as to achieve effective real noise removal.

\textbf{Image denoising with signal priors}:
We hereby review a few representative works on image denoising using signal priors. For a more complete review, we refer the readers to \cite{milanfar2013tour}. In \cite{buades2005non}, Buades {\it et al.} assume that similar image patches recur non-locally throughout an image. Such a self-similarity assumption has been adopted in many subsequent proposals. 
One notable method, block-matching 3-D (BM3D) \cite{dabov2007image}, performs 3-D transform and Wiener filtering on the grouped similar patches. Elad {\it et al.} \cite{elad2006image} propose K-SVD
denoising, which seeks sparse representations to describe noiseless patches with a learned dictionary. 
A very recent work \cite{xu2018trilateral} extents this notion for real image denoising, though its complexity is too high for practical usage.

Graph Laplacian regularization is a recent popular image prior in the literature, {\it e.g.}, \cite{pang2017graph,elmoataz2008nonlocal,gilboa2007nonlocal}. 
Despite its simplicity, graph Laplacian regularization performs reasonably well for many restoration tasks~\cite{milanfar2013tour}. 
It assumes that the original image, denoted as ${\bf x}\in\mathbb{R}^m$, is smooth with respect to an appropriately chosen graph ${\cal G}$. 
Specifically, it imposes that the value of ${\bf x}^{\rm T}{\bf L}{\bf x}$, {\it i.e.}, the graph Laplacian regularizer, should be small for the original image ${\bf x}$, where ${\bf L}\in \mathbb{R}^{m\times m}$ is the Laplacian matrix of graph ${\cal G}$. 
Typically, a graph Laplacian regularizer is employed for a quadratic programming (QP) formulation \cite{pang2017graph,hu2016graph,liu2014progressive}. 
Nevertheless, choosing a proper graph for image restoration remains an open question. 
In \cite{elmoataz2008nonlocal,liu2014progressive}, the authors build their graphs from the corrupted image with simple ad-hoc rules; while in \cite{pang2017graph}, Pang {\it et~al.} derive sophisticated principles for building graphs under strict conditions. 
Different from existing works, our DeepGLR framework constructs neighborhood graphs from the CNN outputs, {\it i.e.,} our graphs are built in a {\it data-driven} manner, which learns the appropriate graph connectivity for restoration directly. 
In \cite{shen2016deep,barron2016fast}, the authors also formulate graph Laplacian regularization in a deep learning pipeline; yet unlike ours, their graph constructions are fixed functions, {\it i.e.}, they are not data-driven.

\textbf{Learning with graphs}:
there exist a few works combining tools of graph theory with data-driven approaches. 
In \cite{kipf2017semi,defferrard2016convolutional} and subsequent works, the authors study the notion of convolution on graphs, which enables CNNs to be applied on irregular graph kernels. 
In \cite{turaga2010convolutional}, Turaga~{\it et~al}. let a CNN to directly output edge weights for fixed graphs; while Egilmez~{\it et~al}. \cite{egilmez2017graph} learn the graph Laplacian matrices with a maximum a posteriori (MAP) formulation. 
Our work also learns the graph structure. 
Different from the methodology of existing works, we build the graphs from the learned features of CNN for subsequent regularizations. 

\section{Deep Graph Laplacian Regularization}\label{sec:glr}
We now present our DeepGLR framework integrating graph Laplacian regularization into CNN for real noise removal. 
A graph Laplacian regularization layer is composed of two modules: a graph construction module \cite{pang2017graph} and a QP solver \cite{amos2017optnet}. 
We first present the details of graph Laplacian regularization \cite{pang2017graph,elmoataz2008nonlocal,hu2016graph} as an image prior, then introduce its encapsulation as a layer in a CNN.

\subsection{Formulation}
\label{ssec:glr_general}
%
We start our illustration with a simple AWGN denoising formulation, which will be extended to take account for more complex cases (Section\,\ref{ssec:iter}). 
Consider the following image corruption model:
\begin{equation}\label{eq:corrupt}
{\bf y} = {\bf x} + {\bf n},
\end{equation}
Here ${\bf x}\in\mathbb{R}^m$ is the original image or image patch (in vector form) with $m$ pixels, while ${\bf n}$ is an additive Gaussian noise term and ${\bf y}$ is the noisy observation. 
Given an appropriate neighborhood graph ${\cal G}$ with $m$ vertices representing the pixels, graph Laplacian regularization assumes the original image ${\bf x}$ is \emph{smooth} with respect to ${\cal G}$ \cite{shuman2013emerging}. 
Denoting the edge weight connecting pixels $i$ and $j$ as $w_{ij}$, the adjacency matrix ${\bf A}$ of graph ${\cal G}$ is an $m$-by-$m$ matrix, whose $(i,j)$-th entry is $w_{ij}$. 
The degree matrix of ${\cal G}$ is a diagonal matrix ${\bf D}$ whose $i$-th diagonal entry is $\sum\nolimits_{j=1}^{m}{w}_{ij}$. 
Then the (combinatorial) graph Laplacian matrix ${\bf L}$ is a positive semidefinite (PSD) matrix given by ${\bf L} = {\bf D} - {\bf A}$, which induces the graph Laplacian regularizer ${\bf x}^{\rm T}{\bf L}{\bf x}\ge 0$ \cite{shuman2013emerging}. 

To recover ${\bf x}\in\mathbb{R}^m$ with graph Laplacian regularization, one can formulate a \textit{maximum a posteriori} (MAP) problem as follows:
\begin{equation}\label{eq:qp_glr}
{\bf x}^{\star} = \mathop {\arg \min }\limits_{\bf{x}} {{\left\| {  {\bf y} - {\bf x} } \right\|}_2^2} + \mu\bm\cdot{\bf x}^{\rm T}{\bf Lx},
\end{equation}
where the first term is a fidelity term (negative log likelihood) computing the difference between the observation ${\bf y}$ and the recovered signal ${\bf x}$, and the second term is the graph Laplacian regularizer (negative log signal prior). 
$\mu\ge 0$ is a weighting parameter. 
For effective regularization, one needs an appropriate graph $\cal G$ reflecting the image structure of ground-truth ${\bf x}$. 
In most works such as \cite{pang2017graph,hu2016graph,osher2017low}, it is derived from the noisy $\bf{y}$ or a pre-filtered version of $\bf{y}$. 

For illustration, we define a matrix-valued function ${\bf F}({\bf y}): \mathbb{R}^m\mapsto\mathbb{R}^{m\times N}$, where its $n$-th column is denoted as ${\bf f}_n$ where ${\bf f}_n\in\mathbb{R}^m$, $1\le n\le N$. 
Hence, applying ${\bf F}$ to observation ${\bf y}$ maps it to a set of $N$ length-$m$ vectors $\{{\bf f}_n\}_{n=1}^{N}$. 
Using the same terminology in \cite{pang2017graph}, the ${{\bf f}_n}$'s are called \emph{exemplars}.
Then the edge weight $w_{ij}$ ($1\le i, j \le m$) is computed by:
%
\begin{equation}\label{eq:weight_ori}
w_{ij} = \exp\left(-\frac{{\rm dist}(i,j)}{2\epsilon^2}\right),
\end{equation}
where
\begin{equation}\label{eq:weight}
{\rm dist}(i,j) = \sum_{n = 1}^N{\left ( {\bf f}_n(i) - {\bf f}_n(j) \right )^2}.
\end{equation}
Here ${\bf f}_n(i)$ denotes the $i$-th element of ${\bf f}_n$. 
\eqref{eq:weight} is the Euclidean distance between pixels $i$ and $j$ in the $N$-dimension feature space defined by $\{{\bf f}_n\}_{n=1}^{N}$. 
In practice, the ${{\bf f}_n}$'s should reflect the characteristics of the ground-truth image ${\bf x}$ for effective restoration. 
Though different works use different schemes to build a similarity graph $\cal G$, most of them differ only in the choice of exemplars ${\bf F}({\bf y})$ (or the ${\bf f}_n$'s). 
In \cite{liu2014progressive,kheradmand2014general}, the authors restrict the graph structure to be a 4-connected grid graph and let ${\rm dist}(i,j)= \left ( {\bf y}(i) - {\bf y}(j) \right )^2$, which is equivalent to let ${\bf F}({\bf y})={\bf y}$. 
In \cite{hu2016graph}, Hu~{\it et al}. operate on overlapping patches and let ${\bf F}({\bf y})$ be the noisy patches similar to ${\bf y}$. 
Pang~{\it et al.} \cite{pang2017graph} interpret the $\{{\bf f}_n\}_{n=1}^{N}$ as samples on a high-dimensional Riemannian manifold and derive the optimal ${\bf F}$ under certain assumptions. 

\subsection{Graph Laplacian Regularization Layer}
\label{ssec:glr_layer}
In contrast to existing works, we deploy graph Laplacian regularization as a layer in a deep learning pipeline, by \emph{implementing the function {\bf F} with a CNN}. 
In other words, the corrupted observation ${\bf y}$ is fed to a CNN (denoted as $\rm{CNN}_{\bf F}$) which outputs $N$ exemplars (or feature maps) $\{{\bf f}_n\}_{n=1}^{N}$. 

Specifically, we perform denoising on a \emph{patch-by-patch} basis, similarly done in \cite{pang2017graph,hu2016graph,liu2014progressive}. 
Suppose the observed noisy image, denoted as ${\cal Y}$, is divided into $K$ overlapping patches $\{{\bf y}_k\}_{k=1}^K$.
Instead of na\"{i}vely feeding each patch to $\rm{CNN}_{\bf F}$ individually then performing optimization, we feed the whole noisy image ${\cal Y}$ to it, leading to $N$ exemplars images of the same size as ${\cal Y}$, denoted as $\{{\cal F}_n\}_{n=1}^N$. 
By doing so, for the $\rm{CNN}_{\bf F}$ with receptive field size as $r$, each pixel $i$ on ${\cal F}_n$ is influenced by all the pixels $j$ on image ${\cal Y}$ if $j$ is in the $r\times r$ neighborhood of $i$. 
As a result, for a larger receptive field $r$, the exemplar ${\cal F}_n$ effectively takes into account more \emph{non-local} information for denoising, resembling the notion of non-local means (NLM) in the classic works \cite{buades2005non,dabov2007image}. 

\begin{figure*}[t]
\centering
    \includegraphics[width=370pt]{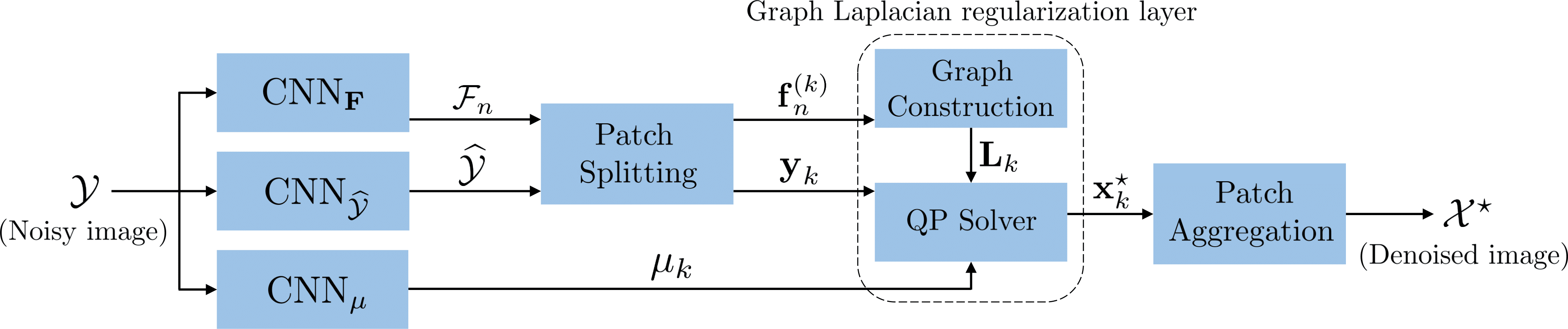}
\caption{Block diagram of the proposed GLRNet which employs a graph Laplacian regularization layer for image denoising.}
\label{fig:glrnet}
\end{figure*}
%
With the exemplar images, we simply divide each of them, say, ${\cal F}_n$, into $K$ overlapping patches ${\bf f}_n^{(k)}\in\mathbb{R}^m$, $1\le k\le K$. 
To denoise a patch ${\bf y}_k$, we build a graph ${\cal G}_k$ with its corresponding $N$ exemplars $\{{\bf f}_n^{(k)}\}_{n=1}^N$ in the way described in Section~\ref{ssec:glr_general}, leading to the graph Laplacian matrix ${\bf L}_k$. 
Rather than a fully connected graph, we choose the \emph{8-connected pixel} adjacency graph structure, {\it i.e.}, in the graph ${\cal G}_k$, every pixel $i$ is only connected to its 8 neighboring pixels. 
Hence, the graph Laplacian ${\bf L}_k$ is sparse with \emph{fixed} sparsity pattern.
The graph Laplacian ${\bf L}_k$, together with patch ${\bf y}_k$, are passed to the QP solver, which resolves the problem \eqref{eq:qp_glr} and outputs the denoised patch ${\bf x}_k^{\star}$.
By equally aggregating the denoised image patches ${\bf x}_k^{\star}$ ($1\le k\le K$), we arrive at the denoised image (denoted by ${\cal X}^{\star}$).
From spectral graph theory \cite{chung1997spectral}, the graph Laplacian regularization layer is always an \emph{adaptive linear low-pass filter}, regardless of the training data.

Apart from the aforementioned procedure, for practical restoration with the graph Laplacian regularization layer, the following ingredients are also adopted.
\begin{enumerate}[(i)]
\item{Generation of $\mu$}: in \eqref{eq:qp_glr}, $\mu$ trades off the importance between the fidelity term and the graph Laplacian regularizer. 
To generate the appropriate $\mu$'s for regularization, we build a light-weight CNN (denoted as $\rm{CNN}_{\mu}$).
Particularly, based on the corrupted image ${\cal Y}$, it produces a set of $\{\mu_k\}_{k=1}^K$ corresponding to the patches $\{{\bf y}_k\}_{k=1}^K$. 
\item{Pre-filtering}: in many denoising literature ({\it e.g.}, \cite{milanfar2013tour,chatterjee2012patch,pang2015optimal}), it is popular to perform a pre-filtering operation to the noisy image ${\cal Y}$ before optimization. 
We borrow this idea and implement a pre-filtering step with a light-weight CNN (denoted as ${\rm CNN}_{\widehat{\cal Y}}$). 
It operates on image ${\cal Y}$ and outputs the filtered image ${\widehat{\cal Y}}$. 
Hence, instead of $\{{\bf y }_k\}_{k=1}^K$, we employ the patches of ${\widehat{\cal Y}}$, {\it i.e.}, $\{\widehat{{\bf y}}_k\}_{k=1}^K$, in the data term of problem \eqref{eq:qp_glr}.
%
\end{enumerate}

We call the presented architecture which performs restoration with a graph Laplacian regularization layer \emph{GLRNet}.
Figure\;\ref{fig:glrnet} shows its block diagram, where the graph Laplacian regularization layer is composed of a \emph{graph construction} module generating graph Laplacian matrices, and a \emph{QP solver} producing denoised patches. The denoised image ${\cal X}^{\star}$ is obtained by aggregating the denoised patches.
We see that, to achieve denoising, a noisy image first goes through the pre-filtering network ${\rm CNN}_{\widehat{\cal Y}}$. 
It is then processed by the graph Laplacian regularization layer, a linear low-pass filter.
As a result, our denoising framework is {\it less sensitive} to the training data. 
Moreover, it is \emph{less affected} by the chosen structure of ${\rm CNN}_{\widehat{\cal Y}}$, as to be seen in Section\,\ref{ssec:samlldata}.

Since the graph construction involves only elementary functions such as exponentials, powers and arithmetic operations, it is differentiable. 
Furthermore, from \cite{amos2017optnet} the QP solver is also differentiable with respect to its inputs. 
Hence, the graph Laplacian regularization layer is fully differentiable, and our denoising pipeline can be {\it end-to-end} trained. 
The backward computation of the graph Laplacian regularization layer, including both the graph construction and the QP solver, is provided in the supplementary material.

\subsection{Iterative Filtering}\label{ssec:iter}
\begin{figure*}[t]
\centering
    \includegraphics[width=320pt]{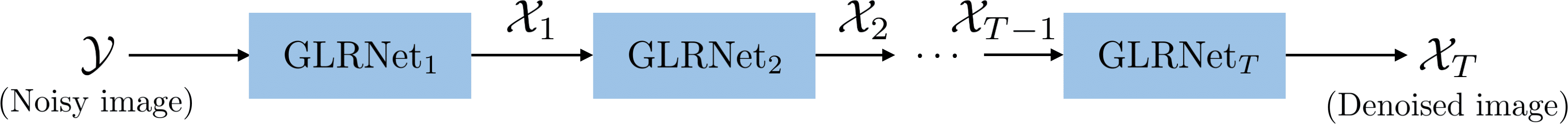}
\caption{Block diagram of the overall DeepGLR framework.}
\label{fig:deepglr}
\end{figure*}
To handle the non-Gaussian property of real image noise, we cascade $T$ blocks of GLRNet (each block has a graph Laplacian regularization layer) for effective restoration, leading to our $DeepGLR$ framework.
Suppose each noise component removed by a GLRNet follows Gaussian distribution, then the overall noise that is removed forms a mixture of Gaussian. 
Ideally, it can approximate any distribution arbitrarily well \cite{reynolds2015gaussian}. 
Moreover, classic literature, {\it e.g.}, \cite{elad2006image,milanfar2013tour,dabov2007image}, also filter the noisy image iteratively to gradually enhance the image quality. 
Similar to \cite{vemulapalli2016deep}, all the GLRNets in our work have the same structure and share the same parameters. 
Figure\;\ref{fig:deepglr} shows the block diagram of DeepGLR. 
In Figure\;\ref{fig:deepglr} and the following presentation, we have removed the superscript ``$\star$'' from ${\cal X}^{\star}$ for simplicity.

To effectively train the proposed DeepGLR framework, we adopt a loss penalizing differences between the recovered image and the ground-truth. 
Given the noisy image ${\cal Y}$, its corresponding ground-truth image ${\cal X}^{(\rm gt)}$ and the restoration result ${\cal X}_T$, our loss function is defined as the mean-square-error (MSE) between ${\cal X}^{(\rm gt)}$ and ${\cal X}_T$, {\it i.e.}, 
\begin{equation}\label{eq:loss_res}
L_{\rm res}\hspace{-2pt}\left({\cal X}^{(\rm gt)}, {\cal X}_T\right)\hspace{-2pt}=\hspace{-2pt}\frac{1}{HW}\sum_{i = 1}^H\hspace{-1pt}\sum_{j = 1}^W\hspace{-2pt}{\left ( {\cal X}^{\rm (gt)}(i,j)\hspace{-2pt}-\hspace{-2pt}{\cal X}_T(i,j) \right)\hspace{-1pt}^2},
\end{equation}
where $H$ and $W$ are the height and width of the images, respectively. ${\cal X}^{\rm (gt)}(i,j)$ is the $(i,j)$-th pixel of ${\cal X}^{\rm (gt)}$, the same for ${\cal X}_T(i,j)$. 
Note that in our experiments, the restoration loss is only applied to the output of the last cascade ${\cal X}_T$, {\it i.e.}, only the final restoration result is supervised.

For simplicity, we have presented our framework for denoising of 1-channel images. 
To adapt it for denoising of color images, the first layers of the CNNs are changed to take 3-channel inputs, and ${\rm CNN}_{\widehat{\cal Y}}$ should output 3 channels. 
Moreover, in the graph Laplacian regularization layer, the 3 channels share the \emph{same} graph for utilizing inter-color correlation; while the QP solver solves for three separate systems of linear equations then outputs a color image. 
We choose to work in the YUV color space.
During training, the loss function of the 3 channels are computed. 
We then take the average as the total loss.

\subsection{Numerical Stability}
\label{ssec:glr_loss}
We hereby analyze the stability of the proposed GLR layer which is indispensable for the stability of the entire framework.
Our denoising approach embedding the QP solver into the processing pipeline have numerical stability guarantee.
Firstly, the problem \eqref{eq:qp_glr} essentially boils down to solving a system of linear equations
\begin{equation}\label{eq:qp_system}
\left({\bf I}+\mu{\bf L}\right){\bf x}^{\star}={\bf y},
\end{equation}
where {\bf I} is an identity matrix. 
It admits a closed-form solution ${\bf x}^{\star} = \left({\bf I} + \mu{\bf L}\right)^{-1}{\bf y}$. 
Thus, one can interpret ${\bf x}^{\star}$ as a filtered version of noisy input {\bf y} with linear filter $\left({\bf I} + \mu{\bf L}\right)^{-1}$. 
As a combinatorial graph Laplacian, ${\bf L}$ is positive semidefinite and its smallest eigenvalue is 0 \cite{shuman2013emerging}.
Therefore, with $\mu \ge 0$, the matrix ${\bf I} + \mu{\bf L}$ is always \emph{invertible}, with the smallest eigenvalue as $\lambda_{\rm min}=1$. 
However, the linear system becomes unstable for a numerical solver if ${\bf I} + \mu{\bf L}$ has a large \textit{condition number} $\kappa$---the ratio between the largest and the smallest eigenvalues $\lambda_{\max} / \lambda_{\min}$ for a normal matrix, assuming an $l_2$-norm \cite{horn1990matrix}. 
Using eigen-analysis, we have the following theorem regarding $\kappa$.
\begin{theorem}\label{thm:cond}
The condition number $\kappa$ of ${\bf I} + \mu{\bf L}$ satisfies
\begin{equation}
\kappa \le 1 + 2 \, \mu \, d_{\rm max}, 
\end{equation}
where $d_{\rm max}$ is the maximum degree of the vertices in ${\cal G}$.
\end{theorem}
\begin{proof}
As discussed, we know $\lambda_{\min}=1$. 
By applying the Gershgorin circle theorem \cite{varga2010gervsgorin}, $\lambda_{\max}$ can be upper-bounded as follows. 
First, the $i$-th Gershgorin disc of $\bf{L}$ has radius $r_i = \sum\nolimits_{j \neq i} |w_{ij}| \le d_{\rm max}$, and the center of the disc $i$ for ${\bf I} + \mu{\bf L}$ is $1+\mu r_i$. 
From the Gershgorin circle theorem, the eigenvalues of ${\bf I} + \mu{\bf L}$ have to reside in the union of all Gershgorin discs. 
Hence, $\lambda_{\rm max} \le \max_{i} \{1 + 2 \, \mu \, r_i\}$, leading to $\kappa = \lambda_{\max} \le 1 + 2 \, \mu \, d_{\max}$.
\end{proof}

Thus, by constraining the value of the weighting parameter $\mu$, we can suppress the condition number $\kappa$ and ensure a stable denoising filter. 
Denote the maximum allowable condition number as $\kappa_{\rm max}$ where we impose $1+2\mu d_{\rm max}\le \kappa_{\max}$, leading to
\begin{equation}
\mu \le \frac{\kappa_{\rm max} - 1}{2d_{\rm max}}=\mu_{\rm max}.
\end{equation}
Hence, if ${\rm CNN}_{\mu}$ generates a value $\mu$ no greater than $\mu_{\rm max}$, then $\mu$ stays unchanged, otherwise it is truncated to $\mu_{\rm max}$. 
We empirically set $\kappa_{\rm max}=250$ for both training and testing to guarantee the stability of our framework.

\section{Experimental Results}\label{sec:results}
\begin{figure*}[t]
\centering
    \includegraphics[width=370pt]{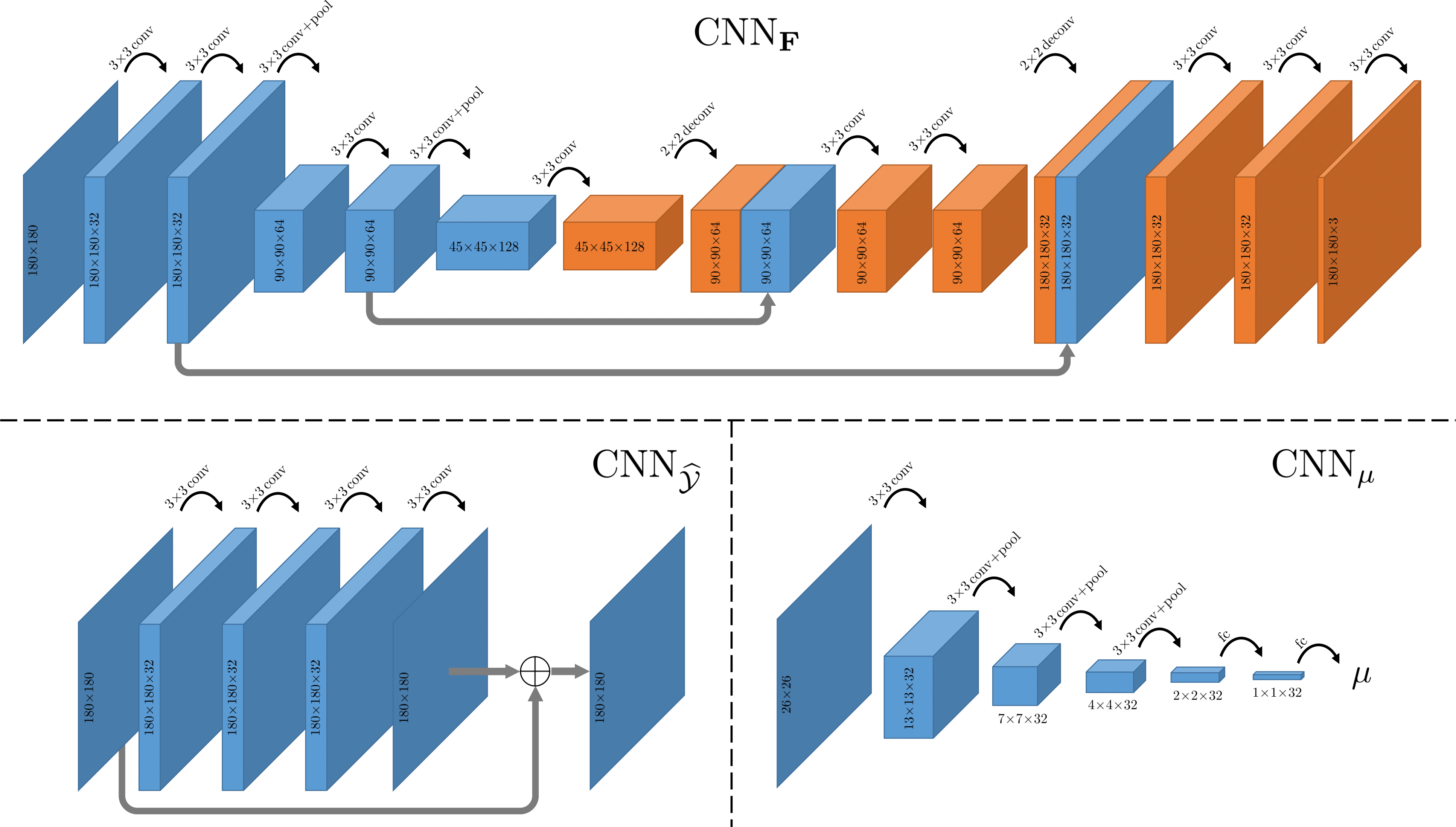}
\caption{Network architectures of ${\rm CNN}_{\bf F}$, ${\rm CNN}_{\widehat{\cal Y}}$ and ${\rm CNN}_{\mu}$ in the experiments. Data produced by the decoder of ${\rm CNN}_{\bf F}$ is colored in orange. }
\label{fig:network}
\vspace{-3pt}
\end{figure*}
\deflen{gauwidth}{82pt}
\begin{figure*}[t]
\centering
    \includegraphics[width=320pt]{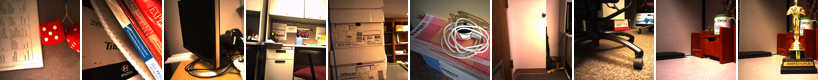}
\caption{The 10 scenes of the RENOIR dataset \cite{anaya2018renoir} used for real image denoising.}
\label{fig:realall}
\vspace{-3pt}
\end{figure*}

Extensive experiments are presented in this section. 
We first describe our adopted CNN architecture and experimental setup in details then apply our model on real image denoising and validate its robustness.
First, it provides satisfactory results when trained with very small amount of data.
Moreover, we demonstrate the strong generalization power of our proposal, which outperforms the state-of-the-art approaches by a remarkable margin.
We use \emph{peak signal-to-noise ratio} (PSNR) which computes in logarithmic (dB) scale as a measurement for objective evaluation. 

%
\subsection{Network Architectures}

Our framework does not limit the choices of network architectures，and one has the freedom in designing the specifications of ${\rm CNN}_{\bf F}$, ${\rm CNN}_{\mu}$ and ${\rm CNN}_{\widehat{\cal Y}}$. 
In our experiment, we choose the networks shown in Figure\;\ref{fig:network}.
Specifically, 
\begin{enumerate}[(i)]
\item{${\rm CNN}_{\bf F}$}: 
To generate exemplars $\{{\bf f}_n\}_{n=1}^N$, we adopt the popular hour-glass structure for ${\rm CNN}_{\bf F}$ which has an encoder and a decoder with skip-connections \cite{ronneberger2015u}. 
Similar to \cite{pang2017graph}, we use $N=3$ exemplars to build the graphs.
\item{${\rm CNN}_{\widehat{\cal Y}}$}: The pre-filtered image $\widehat{\cal Y}$ is simply generated by a light-weight CNN with 4 convolution layers using a residual learning structure \cite{he16}.
\item{${\rm CNN}_{\mu}$}: The weighting parameter $\mu$ is estimated on a patch-by-patch basis. Our experiments uses patch size of $26 \times 26$ for denoising. Hence, starting from a noisy patch, it has undergone 4 convolution layers with $2\times 2$ max pooling and 2 fully-connected layers, leading to the parameter $\mu$.
\end{enumerate}
Except for the last convolution layers of ${\rm CNN}_{\bf F}$ and ${\rm CNN}_{\widehat{\cal Y}}$, and the two deconvolution layers of ${\rm CNN}_{\bf F}$, all the rest network layers shown in Figure\,\ref{fig:network} are followed by a ${\rm ReLU}(\bm\cdot)$ activation function. 
Note that the input image can have different sizes as long as it is a multiple of 4. 
For illustration, Figure\;\ref{fig:network} shows the case when the input is of size $180\times 180$.

\deflen{realwidth}{82pt}
\begin{figure*}
\centering
        \subfloat[Ground-truth]{\includegraphics[width=\realwidth]{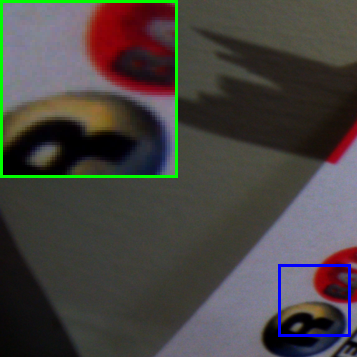}}\hspace{\gauinter}
        \subfloat[Noisy]{\includegraphics[width=\realwidth]{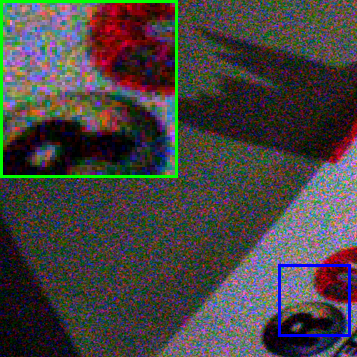}}\hspace{\gauinter}
        \subfloat[CBM3D]{\includegraphics[width=\realwidth]{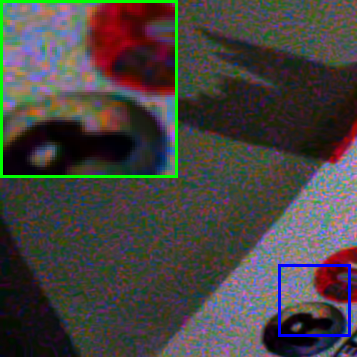}}\hspace{\gauinter}
        \subfloat[MC-WNNM]{\includegraphics[width=\realwidth]{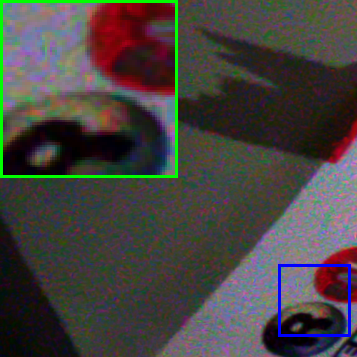}}\hspace{\gauinter}
        \subfloat[Noise Clinic]{\includegraphics[width=\realwidth]{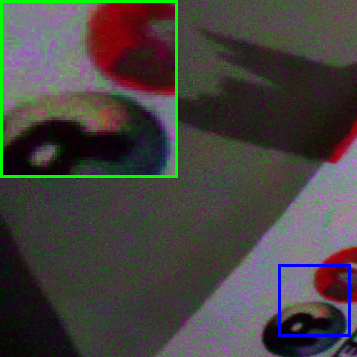}}\\
        \subfloat[CDnCNN]{\includegraphics[width=\realwidth]{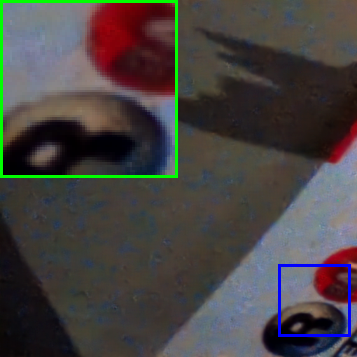}}\hspace{\gauinter}
        \subfloat[GLRNet]{\includegraphics[width=\realwidth]{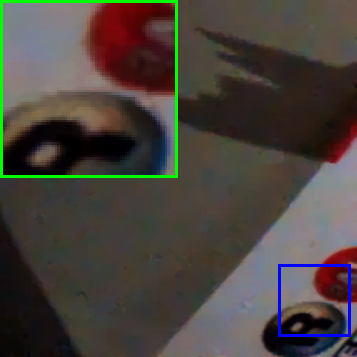}}\hspace{\gauinter}
        \subfloat[DeepGLR-FR]{\includegraphics[width=\realwidth]{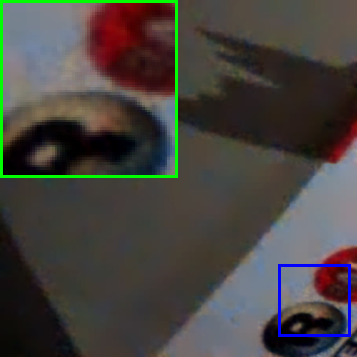}}\hspace{\gauinter}
        \subfloat[DeepGLR-PR]{\includegraphics[width=\realwidth]{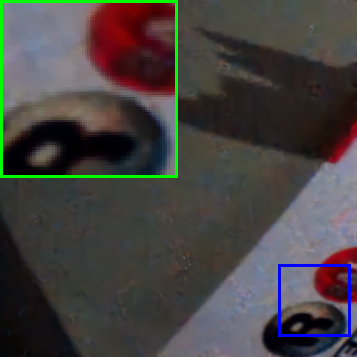}}\hspace{\gauinter}
        \subfloat[DeepGLR]{\includegraphics[width=\realwidth]{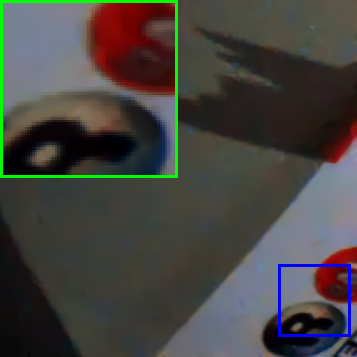}}       
 \caption{Real image noise removal for image 35 of the RENOIR dataset with different approaches.}
 \label{fig:renoir35}
 \vspace{-3pt}
\end{figure*}
\begin{table*}
  \centering
  \caption{Evaluation of different methods for real image denoising. The best results for each metric, except for those tested on the training set, are highlighted in boldface.}
  \footnotesize
    \begin{tabular}{c||c||c|c|c|c|c|c|c}
    \hline
    \multicolumn{1}{c||}{\multirow{2}[4]{*}{Metric}} & \multirow{2}[4]{*}{Noisy} & \multicolumn{7}{c}{Method} \\
\cline{3-9}    \multicolumn{1}{c||}{} &       & CBM3D & MC-WNNM & Noise Clinic & \multicolumn{1}{p{7.21em}|}{CDnCNN\,(Train)} & \multicolumn{1}{p{7.165em}|}{CDnCNN\,(Test)} & \multicolumn{1}{p{7.21em}|}{DeepGLR\,(Train)} & \multicolumn{1}{p{7.375em}}{DeepGLR\,(Test)} \\
    \hline
    \hline
    PSNR  & 20.63 & 26.08 & 26.23 & 27.43 & 34.82 & 32.79 & 34.28 & \textbf{32.96} \\
    \hline
    SSIM  & 0.3081 & 0.6727 & 0.6294 & 0.6040 & 0.8852 & 0.8583 & 0.8795 & \textbf{0.8634} \\
    \hline
    \end{tabular}%
  \label{tab:renoir_comp}%
\end{table*}%

\subsection{Robust Denoising with Small Training Set}\label{ssec:samlldata}


To see the robustness of our proposal, we begin with experimenting DeepGLR with small data.
In this experiment, we employ the {\it RENOIR} \cite{anaya2018renoir} dataset, which consists of real low-light noisy images with the corresponding (almost) noise-free versions. 
Specifically, its subset with 40 scenes collected with a Xiaomi Mi3 smart-phone are used in our experiments.
Since some of the scenes have very low intensities while some of the given ground-truth images are still noisy, we remove the scenes whose ground-truths have: (a) average intensities lower than 0.3 (assuming the intensity ranges from 0 to 1); and (b) estimated PSNRs (provided by \cite{anaya2018renoir}) lower than 36\,dB, leading to 10 valid image pairs. 
Thumbnails of the images are shown in Figure\,\ref{fig:realall}. 

We adopt a two-fold cross validation scheme to evaluate the performance of our approach on small dataset. 
In each of the two trials, we perform training on one fold---only five images---and testing on the other, then measure the performance by the averaging of the results of both trials. 
The 10 images are randomly split into two folds and we repeat such two-fold cross validation process for five times then the results are averaged. 
For objective evaluation, peak signal-to-noise ratio (PSNR) and structural similarity (SSIM)\,\cite{wang2004image} are employed. 
During the training phase, the noisy images, accompanied with their noise-free versions, are fed to the network for training.
For both training and testing, the overlapping patches are of size $26 \times 26$, {\it i.e.}, $m = 26^2 = 676$, where neighboring patches are of a stride 22 apart. 
We let the batch size be 4 and the model is trained for 200 epochs. 
A multi-step learning rate decay policy, with values $[1, 0.5, 0.1, 0.05, 0.01, 0.005]\times 10^{-3}$, are used, where the learning rate decreases at the beginning of epochs $[2, 5, 20, 50, 150]$. 
We implement the network with TensorFlow \cite{tensorflow} on an Nvidia GeForce GTX Titan X GPU. Note that the QP solver is implemented with the TensorFlow layer, \textit{i.e.}, {\tt matrix\_solve\_ls}, for solving a system of linear equations in the least squares sense.

Our DeepGLR is compared with the following approaches: (a)\,CBM3D dedicated for Gaussian noise removal on color image \cite{dabov2007image}\footnote{For testing with CBM3D, we estimate the equivalent noise variances using the ground-truth and the noisy images.}; (b)\,MC-WNNM designed for real image noise removal \cite{xu2017multi}; (c)\,Noise clinic \cite{lebrun2015noise} also designed for real image noise removal; and (d)\,CDnCNN \cite{zhang2017beyond}, a data-driven approach trained with the same dataset as ours. Evaluation results are shown in Table \ref{tab:renoir_comp}, where DeepGLR outperforms competing schemes by a range of 0.17--6.88 dB. 
More visual results are demonstrated in Figure\,\ref{fig:renoir35}, where competing schemes fail to fully remove the noise, while DeepGLR is more satisfactory.
To see the gap between training and testing, performance on the training set is also measured as shown in columns CDnCNN\,(Train) and DeepGLR\,(Train) in Table \ref{tab:renoir_comp}, where CDnCNN excels in training set but not in testing set indicating a strong overfitting. 
This is because: 
\begin{enumerate}[(i)]
\item Only \emph{5 images} are available for training in this experiment, letting CDnCNN strongly \emph{overfit} to the training data. However, our DeepGLR is less sensitive to the deficiency of the training data.
\item While CDnCNN is most suitable for Gaussian noise removal (as stated in \cite{zhang2017beyond}), our DeepGLR adaptively learns the suitable graphs to low-pass filter the real noisy image, which weakens the impact of the complex real noise statistics.
\end{enumerate}

\begin{figure*}[t]
\centering
        \subfloat{\includegraphics[width=\realwidth]{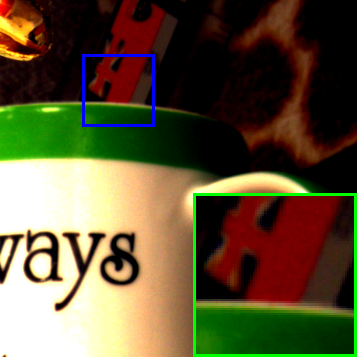}}\hspace{\gauinter}
        \subfloat{\includegraphics[width=\realwidth]{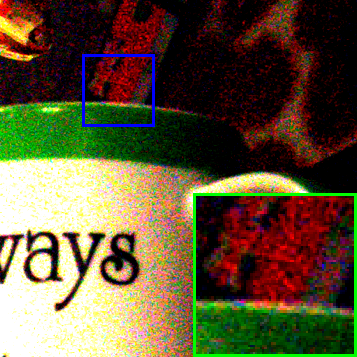}} \hspace{\gauinter}
        \subfloat{\includegraphics[width=\realwidth]{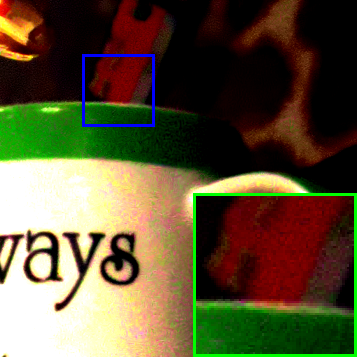}}\hspace{\gauinter}
        \subfloat{\includegraphics[width=\realwidth]{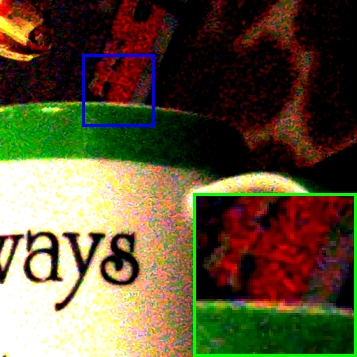}}\hspace{\gauinter}
        \subfloat{\includegraphics[width=\realwidth]{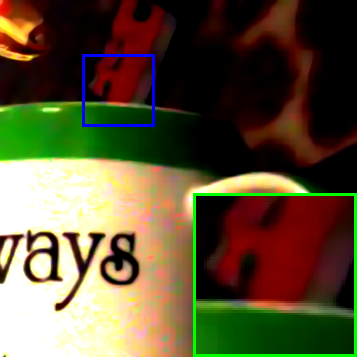}}\\
        \vspace{-10pt}
        \setcounter{subfigure}{0}
        \subfloat[Ground-truth]{\includegraphics[width=\realwidth]{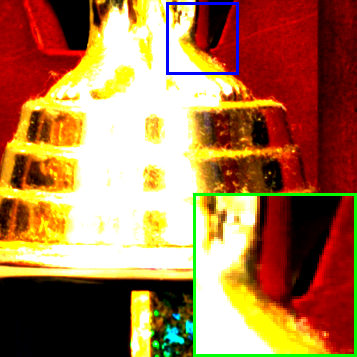}}\hspace{\gauinter}
        \subfloat[Noisy]{\includegraphics[width=\realwidth]{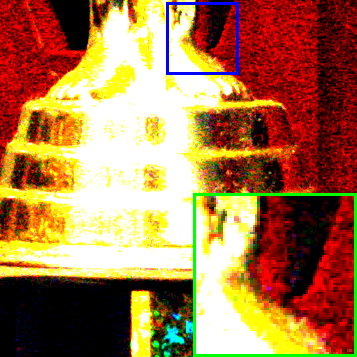}}\hspace{\gauinter}
        \subfloat[Noise Clinic]{\includegraphics[width=\realwidth]{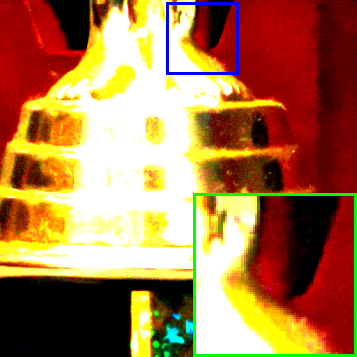}}\hspace{\gauinter}
        \subfloat[CDnCNN]{\includegraphics[width=\realwidth]{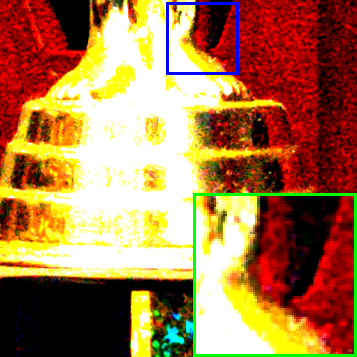}}\hspace{\gauinter}
        \subfloat[DeepGLR]{\includegraphics[width=\realwidth]{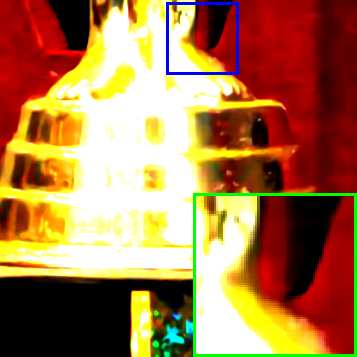}}
        
 \caption{The proposed DeepGLR trained for AWGN denoising generalizes well to real image denoising.}
 \label{fig:gen}
 \vspace{-8pt}
\end{figure*}

To better understand DeepGLR, we also consider several of its variant: (a)\,The pre-filtering network ${\rm CNN}_{\widehat{\cal Y}}$ is removed, we call the resulting method DeepGLR-PR (``PR'' stands for pre-filter removed); (b)\,${\rm CNN}_{\widehat{\cal Y}}$ is replaced by CDnCNN, the resulting method is called DeepGLR-PC (``PC'' stands for pre-filter with CDnCNN); (c)\,${\rm CNN}_{\bf F}$ is removed, and directly use the output of ${\rm CNN}_{\widehat{\cal Y}}$ as exemplars for graph construction, this scheme is referred to as DeepGLR-FR (``FR'' stands for ${\rm CNN}_{\bf F}$ removed); (d)\,GLRNet. 
Evaluations are provided in Table \ref{tab:ablation}, where DeepGLR-PC provides similar performance as DeepGLR, suggesting the GLR layer can perform effective denoising, and adding extra layers to ${\rm CNN}_{\bf F}$ is of little use. 
Moreover, DeepGLR-PC ($33.03$\,dB) can be regarded as a CDnCNN module ($32.79$\,dB) with GLR as the post-processing, indicating that GLR boosts CDnCNN's performance. 
Apart from DeepGLR-PC, the others provide less satisfactory results, which is consistent with results in Figure\,\ref{fig:renoir35}. 
Without the module for exemplar learning (DeepGLR-FR), DeepGLR cannot capture the underlying image structure; without pre-filtering, the GLR layer has limited effect (DeepGLR-PR); without iterative filtering, one GLRNet alone cannot fully remove real noise with complicated statistics. 
In light of this, DeepGLR stands as a composite of modules, each playing an irreplaceable role.

\begin{table}[t]
  \centering 
  \caption{Evaluation of different variants of DeepGLR for real image denoising. The best two results for each metric are highlighted in boldface.}
  \scriptsize
    \begin{tabular}{c||c||c|c|c|c}  
    \hline
    \multirow{3}[4]{*}{Metric} & \multirow{3}[4]{*}{DeepGLR} & \multicolumn{4}{c}{Variant of DeepGLR} \\
\cline{3-6}          &       & DeepGLR & DeepGLR & DeepGLR & GLRNet\\
 &       & -PR & -PC & -FR & \\
\cline{3-6}
    \hline
    \hline
    PSNR  & \textbf{32.9593} & 32.33 & \textbf{33.025} & 32.51 & 32.50 \\
    \hline
    SSIM  & \textbf{0.8634} & 0.8534 & \textbf{0.8637} & 0.8552 & 0.8599 \\
    \hline
    \end{tabular}%
  \label{tab:ablation}%
  \vspace{-9pt}
\end{table}%

\subsection{Cross-Domain Generalization}
\label{ssub:cross}
We hereby evaluate the robustness of our approach in terms of its \emph{cross-domain generalization ability}. 
Specifically, we evaluate on the RENOIR dataset with DeepGLR and CDnCNN trained for AWGN blind denoising. 

During training, we adopts the same training dataset as CDnCNN, \textit{i.e.}, 432 images from BSDS500, and covers the noise level range $\sigma \in [0,55]$. The same network structure, optimizer and learning rate settings are used for AWGN denoising as for real noise removal detailed in Section\,\ref{ssec:samlldata}.

We first evaluate AWGN blind denoising where the results, in terms of PSNR and SSIM, are shown in Table \ref{tab:gaussian}. 
The rest 68 images from BSDS500 are adopted for evaluation with noise standard deviations 15, 25, 50 with CBM3D as the baseline method. DeepGLR has results \emph{on par with} CDnCNN, both of them surpasses CBM3D by a large margin. 
The results of OGLR \cite{pang2017graph} are also shown in Table \ref{tab:gaussian}. 

More importantly, we evaluate the performance of DeepGLR and CDnCNN on RENOIR dataset to investigate their \textbf{cross-domain generalization} ability.
For comparison, we include noise clinic \cite{lebrun2015noise} designed for real noise removal as a baseline method. 
Objective performance, in terms of PSNR and SSIM, are listed in Table\,\ref{tab:crossdomain}.
We see that, DeepGLR has a PSNR performance of 30.10\,dB, \textit{outperforming CDnCNN by 5.74\,dB, and noise clinic by 1.87\,dB}. 
Nevertheless, CDnCNN provides better results in Gaussian noise removal as shown in Table \ref{tab:gaussian}.
This indicates that CDnCNN is \textbf{strongly overfitted} to the case of Gaussian noise removal and fails to generalize to real noise, while DeepGLR provides satisfactory denoising results.

Subjective results demonstrated in Figure\,\ref{fig:gen} show the denoising results of two image fragments from the RENOIR dataset. Noise clinic still has noticable noise unremoved, and CDnCNN almost preserves most noise. 
In contrast, our DeepGLR provides the best visual quality, removing the noise while preserving the sharp edge details. 
The strong domain generalization stems from robust exemplar learning module in capturing the intrinsic image structure with presence of complex noise, boosting the denoising performance through the GLR layer in recovering the clean image.

\begin{table}[htbp]
  \centering \scriptsize
  \caption{Average PSNR (dB) and SSIM values for Gaussian noise removal.}
    \begin{tabular}{c||c|c|c|c}
    \hline\bigstrut[t]
    \multirow{2}[4]{*}{\hspace{-4pt}Noise\hspace{-4pt}} & \multicolumn{4}{c}{Method (PSNR / SSIM)} \bigstrut[t]\\
\cline{2-5}          & CBM3D & CDnCNN & OGLR & DeepGLR \bigstrut[t]\\
    \hline
    \hline
    15    &  33.49 / 0.9216 & 33.80 / 0.9268 & 33.52 / 0.9198 & 33.65 / 0.9259 \bigstrut\\
    \hline
    25    &  30.68 / 0.8675 & 31.13 / 0.8799 & 30.79 / 0.8661 & 31.03 / 0.8797 \bigstrut\\
    \hline
    50    &  27.35 / 0.7627 & 27.91 / 0.7886 & 27.84 / 0.7755 & 27.86 / 0.7924 \bigstrut\\
    \hline
    \end{tabular}%
  \label{tab:gaussian}%
  \vspace{-9pt}
\end{table}%

\begin{table}[htbp]
  \centering\scriptsize
  \caption{Evaluation of cross-domain generalization for real image denoising. The best results are highlighted in boldface.}
    \begin{tabular}{c|c||c|c|c}
    \hline
    \multirow{2}[4]{*}{Metric} & \multirow{2}[4]{*}{Noisy} & \multicolumn{3}{c}{\vspace{-1.0pt}$\phantom{\hat{I}}\mathop{\textrm{Method} }\limits_{\phantom{.}}\phantom{\hat{I}}$} \bigstrut\\
\cline{3-5}          &   & Noise Clinic & CDnCNN & DeepGLR  \bigstrut[t]\\
    \hline
    \hline
    PSNR  & 20.36   & 27.43 & 24.36& \textbf{30.10} \bigstrut[t]\\
    \hline
    SSIM & 0.1823   &0.6040 & 0.5206& \textbf{0.8028}  \bigstrut[t]\\
    \hline
    \end{tabular}%
  \label{tab:crossdomain}%
  \vspace{-8pt}
\end{table}%


\section{Conclusion}\label{sec:conclusion}

In this work, we incorporate graph Laplacian regularization into a deep learning framework for real image noise removal. 
Given a corrupted image, it is first fed to a CNN, then neighborhood graphs are constructed from the CNN outputs on a patch-by-patch basis. 
The graph construction and the denoising process is fully differentiable, hence the overall pipeline can be end-to-end trained. 
We demonstrate the robustness of the proposed framework---DeepGLR---for real noise removal, from two different aspects.
Firstly, it demonstrates higher immunity to overfitting while training with small dataset. 
Moreover, it manifests strong cross-domain generalization ability when training and testing data have different statistics.

\newpage
{\small
\bibliographystyle{ieee}

}


\end{document}